\documentclass[wcp]{jmlr}

\usepackage{longtable}

\usepackage{booktabs,color}

\usepackage{dsfont}
\usepackage{wasysym}

\title[$\ell_{1}$ Regularized GTD Learning]
{
$\ell_{1}$ Regularized Gradient Temporal-Difference Learning 
}

\author{\Name{Dominik Meyer} \Email{dominik.meyer@tum.de} \\ 
        \Name{Hao Shen} \Email{hao.shen@tum.de} \\
        \Name{Klaus Diepold} \Email{kldi@tum.de} \\
        \addr Institute for Data Processing, 
        Technische Universit\"at M\"unchen, Germany
}

\begin{document}

\maketitle
\thispagestyle{plain}

\begin{abstract}
	In this paper, we study the Temporal Difference (TD) learning with linear 
	value function approximation.
	It is well known that most TD learning algorithms are unstable with 
	linear function approximation and off-policy learning. 
	Recent development of \emph{Gradient TD} (GTD)
	algorithms has addressed this problem successfully.
	However, the success of GTD algorithms requires a set of
	well chosen features, which are not always available.
	When the number of features is huge, the GTD algorithms might
	face the problem of overfitting and being computationally expensive.
	To cope with this difficulty, regularization techniques, in 
	particular $\ell_{1}$ regularization, have attracted significant attentions
	in developing TD learning algorithms.
	The present work combines the GTD algorithms with $\ell_{1}$ regularization.
	%
	We propose a family of $\ell_{1}$ regularized GTD algorithms, which employ
	the well known soft thresholding operator.
	We investigate convergence properties of the proposed algorithms, and depict  
	their performance with several numerical experiments. \vspace{2mm}
\end{abstract}


\section{Introduction}
One fundamental problem 
in Reinforcement Learning (RL) is to learn the long-term
expected reward, i.e. the value function, which can consequently 
be used for determining a good control policy, cf. \cite{sutt:book98}.
In the general setting with large or infinite state space, exact representation 
of the actual value function is often inhibitively computationally expensive 
or hardly possible. 
To overcome this difficulty, function approximation techniques 
are employed for estimating the value function from sampled trajectories.
The quality of the learned policy depends significantly on 
the chosen function approximation technique.

In this paper, we consider the technique of \emph{linear value function approximation}.
The value function is represented or approximated as a linear combination of a set 
of features, or basis functions. 
These features are generated from the sampled states via 
either some heuristic constructions, e.g. \cite{brad:ml96,kell:icml06},
or kernel-based approaches, e.g. \cite{tayl:icml09}.
A common approach generates firstly a vast number of features, which  
is often much larger than the number of available samples, and 
then chooses automatically relevant features to approximate the 
actual value function. 
Unfortunately, such approaches may fail completely due to overfitting.
To cope with this situation, regularization techniques are necessarily to be
employed.
Other than the simple $\ell_{2}$ regularization, which penalizes
the smoothness of the learned value function, e.g. \cite{fara:nips08}, 
in this work we focus on $\ell_{1}$ regularization.
The $\ell_{1}$ regularization often produces sparse solutions, thus can serve as 
a method of automatic feature selection for linear value function 
approximation.

This work focuses on the development of Temporal Difference (TD) learning algorithms, 
cf. \cite{brad:ml96}.
Recent active researches on applying $\ell_{1}$ regularization to TD learning have 
led to a various number of effective algorithms, e.g.~\cite{loth:adprl07,kolt:icml09,
john:nips10,geis:ewrl11,hoff:ewrl11}.
It is important to notice that $\ell_{1}$ minimization has been extensively 
studied in the areas of compressed sensing and image processing, and many 
efficient $\ell_{1}$ minimization algorithms have been developed, cf.
\cite{cand:ip07,zibu:spm10}.
Very recently, two advanced $\ell_{1}$ minimization algorithms have been adapted to the 
TD learning, i.e. the Dantzig selector based TD algorithm from \cite{geis:icml12} and the
orthogonal matching pursuit based TD algorithm developed in \cite{pain:icml12}.

On the other hand, most TD learning algorithms are known to be unstable with 
linear value function approximation and off-policy learning.
By observing the fact that most original forms of TD algorithms are not
true gradient descent methods, a new class of intrinsic gradient TD (GTD) learning 
algorithms with linear value function approximation are developed and proven 
to be stable, cf. \cite{sutt:nips08,sutt:icml09}. 
%
%
However, it is important to know that success of GTD algorithms
might be limited due to the fact that the GTD family requires a set of well chosen 
features.
In other words, the GTD algorithms are in potential danger of overfitting.
%
%
The key contribution of the present work is the development of a family of
$\ell_{1}$ regularized GTD algorithms, referred to as \emph{GTD-IST} algorithms.
Convergence properties of the proposed algorithms are investigated from
the perspective of stochastic optimization.

The paper is outlined as follows. In Section~\ref{sec:02}, we briefly introduce
a general setting of TD learning and provide some preliminaries of TD objective
functions. 
Section~\ref{sec:03} presents a framework of $\ell_{1}$ regularized GTD 
learning algorithms, and investigates their convergence properties.
In Section~\ref{sec:04}, several numerical experiments depict the practical performance
of the proposed algorithms, compared with several existing $\ell_{1}$ regularized 
TD algorithms.
Finally, a conclusion is drawn in Section~\ref{sec:05}.

\section{Notations and Preliminaries}
\label{sec:02}
In this work, we consider a RL process as a Markov Decision Process (MDP), defined as a tuple $(\mathcal{S},
\mathcal{A},P,r,\gamma)$, where
$\mathcal{S}$ is a set of possible states of the environment, $\mathcal{A}$ is a set 
of actions of the agent, $P \colon \mathcal{S} \times \mathcal{A} \times \mathcal{S}
\to [0, 1]$ the conditional transition probabilities $P(s,a,s')$ over state transitions
from state $s$ to state $s'$ given an action $a$, $r \colon \mathcal{S} \to \mathbb{R}$ 
is a reward function assigning immediate reward $r$ to 
a state $s$, and $\gamma \in [0, 1]$ is a discount factor.

\subsection{TD Learning with Linear Function Approximation}
The goal of a RL agent is to learn a mapping from states to actions,
i.e. a \emph{policy} $\pi \colon \mathcal{S} \to \mathcal{A}$, which maximizes the value 
function $V^{\pi} \colon \mathcal{S} \to \mathbb{R}$ of a state $s$ taking a policy 
$\pi$, defined as
\begin{equation}
\label{eq:01}
	V^{\pi}(s) := \mathbb{E}\left[ {\textstyle \sum\limits}_{t=0}^{\infty}\gamma^{t} 
	r(s_{t}) | s_{0}=s, \pi\right].
\end{equation}
It is well known that, for a given policy $\pi$, the value function $V^{\pi}$ fulfills 
the \emph{Bellman equation}, i.e.
\begin{equation}
\label{eq:02}
	V^{\pi}(s) = r(s) + \gamma \sum_{s'} P(s,\pi(s),s') V^{\pi}(s').
\end{equation}
The right hand side of \eqref{eq:02} is often referred to as the 
\emph{Bellman operator} for
policy $\pi$, denoted by $\mathcal{T}V^{\pi}(s)$.
In other words, the value function $V^{\pi}(s)$ is the fixed point of the Bellman operator
$\mathcal{T}V^{\pi}(s)$, i.e. $V^{\pi}(s)  = \mathcal{T}V^{\pi}(s)$.

%
When the state space is too large or infinite, exact representation of the value 
function is often practically unfeasible.
%
%
Function approximation is thus of great demand for estimating the actual value function.
A popular approach is to construct a set of features by the map
$\phi \colon \mathcal{S} \to \mathbb{R}^{k}$, which are called the \emph{features} or 
\emph{basis functions}, and then to approximate the value function by a linear 
function.
Concretely, for a given state $s$, the value function is approximated by 
\begin{equation}
\label{eq:03}
	V(s) \approx (\phi(s))^{\top} \theta =: V_{\theta},
\end{equation}
where $\theta \in \mathbb{R}^{k}$ is a parameter vector.
In the setting of TD learning, the parameter $\theta$ is updated at each time step $t$, 
i.e. for each state transition and the associated reward $(s_{t},r_{t},s_{t}')$.
Here, we consider  the simple one-step TD learning with linear function 
approximation, i.e. $\lambda = 0$ in the framework of TD($\lambda$) learning. 
The parameter $\theta$ is updated as follows
\begin{equation}
\label{eq:td}
	\theta_{t+1} = \theta_{t} + \alpha_{t} \delta_{t} \phi_{t},
\end{equation}
where $\alpha_{t} > 0$ is a sequence of step-size parameters, and $\delta_{t}$ is
the simple TD error
\begin{equation}
	\delta_{t} = r_{t} + \theta_{t}^{\top} \left(\gamma \phi_{t}' - 
	\phi_{t}\right).
\end{equation}
Note, that the TD error $\delta_{t}$ can be considered as a function of the
parameter $\theta_{t}$. By abuse of notation, in the rest of the paper 
we also denote $\delta_{\theta} = 
\delta(\theta) := r + \theta^{\top} \left(\gamma \phi' - \phi\right)$.

\subsection{Three Objective Functions for TD Learning}
In order to find an optimal parameter $\theta^{*}$ via an optimization process, 
one has to define an appropriate objective function, which accurately measures the 
correctness of the current value function approximation, i.e. how far the current 
approximation is away from the actual TD solution.
%
%
In this subsection, we recall three popular objective functions for TD learning.

Motivated by the fact that the value function is the fixed point of the Bellman 
operator for a given policy, correctness of an approximation $V_{\theta}$ can be simply measured by the TD error itself, i.e.
\begin{equation}
	J_{1} \colon \mathbb{R}^{k} \to \mathbb{R}, \qquad 
	J_{1}(\theta) := \tfrac{1}{2}
	\left\| V_{\theta} - \mathcal{T}V_{\theta} \right\|^2_{D}
	= \tfrac{1}{2}\left( \mathbb{E}[\delta_{\theta}] \right)^{2},
\end{equation}
where $D \in \mathbb{R}^{|\mathcal{S}| \times |\mathcal{S}|}$ is a diagonal 
matrix, whose components are some state distribution.
This cost function is often referred to as the \emph{Mean Squared Bellman Error} 
(MSBE).
Ideally, the minimum of the MSBE function admits a good value 
function approximation.
Unfortunately, it is well known that, in practice, the performance of an approximation
$V_{\theta}$ depends on the pre-selected feature space $\mathcal{H} := \left\{ \Phi 
\theta | \theta \in \mathbb{R}^{k} \right\}$, i.e. the span of the features 
$\Phi := \phi(\mathcal{S})$.
By introducing the projector as
\begin{equation}
	\Pi = \Phi \big( \Phi^{\top} D \Phi \big)^{-1} 
	\Phi^{\top} D,
\end{equation}
%
the so-called \emph{Mean Squared Projected Bellman Error} (MSPBE) is ofter preferred 
\begin{equation}
\label{eq:mspbe}
\begin{split}
	J_{2} \colon \mathbb{R}^{k} \to \mathbb{R}, \qquad 
	J_{2}(\theta) := &~\! \tfrac{1}{2}\left\| V_{\theta} - \Pi 
	\mathcal{T}V_{\theta} \right\|^2_{D} \\
	= &~\! \tfrac{1}{2}\mathbb{E}[\delta_{\theta } \phi]^{\top} 
	\mathbb{E}[\phi \phi^{\top}]^{-1} \mathbb{E}[\delta_{\theta} \phi].
\end{split}
\end{equation}
Minimizing the MSPBE function finds a fixed point of the projected Bellman operator
in the feature space $\mathcal{H}$, i.e. $V_{\theta}  = \Pi \mathcal{T}V_{\theta}$.

Finally, we present a less popular objective function for TD learning.  
Recall the TD parameter update as defined in \eqref{eq:td}.
The vector $\mathbb{E}[\delta_{\theta} \phi] \in \mathbb{R}^{k}$ in the second 
summand can be considered as an error for a given $\theta$.
It is expected to be equal to zero at the TD solution.
Hence, one can use the $\ell_{2}$ norm of this vector, defined as
\begin{equation}
\label{eq:neu}
	J_{3} \colon \mathbb{R}^{k} \to \mathbb{R}, \qquad 
	J_{3}(\theta) = \tfrac{1}{2} \mathbb{E}[\delta_{\theta}\phi]^{\top} 
	\mathbb{E}[\delta_{\theta}\phi],
\end{equation}
as an objective function for TD learning.
The function $J_{3}$ is referred to as the \emph{Norm of Expected TD Update} (NEU), 
which is used to derive the original GTD algorithm in \cite{sutt:nips08}.

\section{Stochastic Gradient Algorithms for $\ell_{1}$ Regularized TD Learning}
\label{sec:03}
%
%
In the first part of this section, we present a general framework of
gradient algorithms for minimizing the $\ell_{1}$
regularized TD objective functions.
The second subsection develops two $\ell_{1}$ regularized 
stochastic gradient TD algorithms in the online setting, 
and investigates their convergence properties from the perspective of 
stochastic optimization.

\subsection{$\ell_{1}$ Regularized TD Learning}
\label{sec:31}
Applying an $\ell_{1}$ regularizer to the parameter $\theta$ leads to the following 
objective function
\begin{equation}
\label{eq:obj} 
	F_{i}(\theta) := J_{i}(\theta) + \eta \|\theta\|_{1},
\end{equation}
where $i \in \{1,2,3\}$ and $\|\theta\|_{1} = \sum_{i} |\theta_{i}|$ denotes the $\ell_{1}$ norm of a vector $\theta = [\theta_{1}, \ldots, \theta_{k}]^{\top} \in
\mathbb{R}^{k}$.
Here, the scalar $\eta > 0$ weighs the regularization term $\| \theta \|_{1}$, and 
balances the sparsity of $\theta$ against the TD objective function $J_{i}$.
%
%
The \emph{iterative soft thresholding} (IST) algorithm is nowadays one classic 
algorithm for minimizing the cost function \eqref{eq:obj}.
It can be interpreted as an extension of the classical gradient algorithm.
Due to its high popularity, we skip the derivation of the IST algorithm, and refer to \cite{zibu:spm10} and the references therein for further reading.

%
Given $x \in \mathbb{R}^{m}$ and $\nu > 0$, the \emph{soft thresholding operator} 
applied to $x$ is defined as
\begin{equation}
\begin{split}
	\Psi_{\nu}(x) :=~\! & \operatorname{sgn}(x)
	\astrosun \operatorname{max}\{|x| - \nu, 0\} \\
	=~\! & \left\{\!\!
	\begin{array}{ll}
		x - \operatorname{sgn}(x)\nu,\qquad & \text{if}~|x|>\nu, \\
		0, & \text{otherwise},
	\end{array}
	\right.
\end{split}
\end{equation}
where $\operatorname{sgn}(\cdot)$ and
$\max(\cdot)$ are entry-wise,
and $\astrosun$ is the entry-wise multiplication. 
Then, minimization of the objective function \eqref{eq:obj} can be achieved
via applying the soft thresholding operator iteratively.
Straightforwardly, we define the IST based TD update as follows
\begin{equation}
\label{ist_td}
	\theta_{t+1} = \Psi_{\alpha_t \eta} \left( \theta_t - \alpha_t 
	\nabla J_{i}(\theta_{t}) \right),
\end{equation}
where $\alpha_{t} > 0$, and 
$\nabla J_{i}(\theta_{t})$ denotes the gradient update of $J_{i}(\theta_{t})$.
Specifically, the gradient updates of the three objective functions are given 
as 
\begin{equation}
\left\{\!\!
	\begin{array}{l}
		\nabla J_{1}(\theta_{t}) = \mathbb{E}\big[ \delta_{t} \big] 
		\mathbb{E}\big[ (\gamma \phi'_{t} - \phi_{t}) \big], \\
		\nabla J_{2}(\theta_{t}) = \mathbb{E}\big[ (\gamma \phi'_{t} - \phi_{t})
		\phi_{t}^{\top} \big] \left( \mathbb{E}\big[ \phi_{t} \phi_{t}^{\top} \big]
		\right)^{-1} \mathbb{E}\big[ \delta_{t} \phi_{t} \big], \\[0.6mm]
		\nabla J_{3}(\theta_{t}) = \mathbb{E}\big[ (\gamma \phi'_{t} - \phi_{t})
		\phi_{t}^{\top} \big] \mathbb{E}\big[ \delta_{t} \phi_{t} \big].
	\end{array}
\right.
\end{equation}
We refer to this family of algorithms as TD-IST algorithms.
Note that IST has been employed in developing fixed point 
TD algorithms in \cite{pain:tr12}, whereas in this work we focus on developing 
intrinsic gradient TD algorithm.

\subsection{Stochastic GTD-IST Algorithms}
The TD-IST algorithms presented in the previous subsection are only 
applicable in the batch setting. In some real applications, it is certainly
favorable to have them working online. 
%
%
Stochastic gradient descent algorithms can be developed straightforwardly to 
minimize the $\ell_{1}$ regularized TD objective functions.
%

%
%

Now let us consider the online setting, i.e. given a sequence of data samples 
$\phi_{1}, \phi_{2}, \ldots$.
%
%
%
In the form of stochastic gradient descent, we propose a general form of 
parameter update as 
\begin{equation}
	\quad \theta_{t+1} = \Psi_{\alpha_t \eta} \left( \theta_t - \alpha_t 
	\widetilde{\nabla} J_{i}(\theta_{t}) \right),
\end{equation}
where $\widetilde{\nabla} J_{i}(\theta_{t})$ denotes the stochastic gradient updates of 
$J_{i}(\theta_{t})$, or their appropriate stochastic approximations, cf. \cite{sutt:nips08,sutt:icml09}.
To investigate convergence properties of the proposed algorithms requires 
results from \cite{duch:jmlr09}, which develops a general framework for analyzing
empirical loss minimization with regularizations.
We adapt the result in corollary~10 from \cite{duch:jmlr09} to our current 
setting as follows.

\begin{theorem}
\label{thm:01}
	Let the function $J \colon \mathbb{R}^{k} \to \mathbb{R}$ be smooth and 
	strictly convex and $\theta^{*} \in \mathbb{R}^{k}$ be the global minimum
	of the function $F(\theta) := J(\theta) + \eta \| \theta \|_{1}$ with
	$\eta > 0$.
	If the following three conditions hold: 
	(1) $\theta^{*}$ fulfills $\| \theta_{t} - \theta^{*} \|_{2} \le d$ for
	some constant $d > 0$;
	(2) $\| \nabla J (\theta_{t}) \|_{2} \le g$ for
	some constant $g > 0$; and (3) a stochastic estimate of the gradient
	$\widetilde{\nabla} J(\theta_{t})$ fulfills 
	$\mathbb{E}[\widetilde{\nabla} J(\theta_{t})] =
	\nabla J(\theta_{t})$,
	then IST based stochastic algorithms converge with 
	probability one to $\theta^{*}$.
\end{theorem}

%
Let us look at the $\ell_{1}$ regularized NEU function $F_{3}$ first. 
Recall the approximate stochastic gradient update, developed in \cite{sutt:nips08}, as
\begin{equation}
\label{eq:gtd_ist}
	\widetilde{\nabla} J_{3}(\theta_{t}) = (\phi_{t}^{\top} u_{t}) (\gamma \phi_{t}'
	- \phi_{t}),
\end{equation}
with 
\begin{equation}
	u_{t+1} = u_{t} + \beta_{t} (\delta_{t} \phi_{t} - u_{t}),
\end{equation}
where $\beta_{t} > 0$ is a step size parameter.
We refer to the corresponding algorithm as the \emph{GTD-IST} algorithm. 
Convergence properties of the GTD-IST algorithm are characterized in 
the following corollary.
\begin{corollary}
\label{cor:01}
	If $(\phi_{t},r_{t},\phi'_{t})$ is an i.i.d sequence with uniformly
	bounded second moments, and the matrix 
	$\mathbb{E}[\phi (\gamma \phi' - \phi)^{\top}] \in \mathbb{R}^{k \times k}$ 
	is invertible, then the GTD-IST algorithm, whose update is specified in
	\eqref{eq:gtd_ist}, converges with probability one to the TD solution. 
\end{corollary}
\begin{proof}
	Recall the TD error $\delta_{\theta} = r + \theta^{\top} (\gamma \phi' - \phi)$.
	The $\ell_{1}$ regularized NEU cost function $F_{3}$ can be written as
	\begin{equation}
	\begin{split}
		F_{3}(\theta) =~\! & \mathbb{E}[\delta_{\theta } \phi]^{\top} 
		\mathbb{E}[\delta_{\theta} \phi] + \eta \| \theta \|_{1} \\
		=~\! & \mathbb{E}\big[r \phi + \theta^{\top} (\gamma \phi' - \phi) \phi 
		\big]^{\top}
		\mathbb{E}\big[r \phi + \theta^{\top} (\gamma \phi' - \phi) \phi \big]
		+ \eta \| \theta \|_{1}.
	\end{split}
	\end{equation}
	It is easily seen that the regularized function $F_{3}$ is strictly convex 
	if the matrix 
	$\mathbb{E}[\phi (\gamma \phi' - \phi)^{\top}]$ is invertible.
	The TD solution is then the global minimum of $F_{3}$.
	The condition of $(\phi_{t},r_{t},\phi'_{t})$ being an i.i.d sequence 
	with uniformly bounded second moments ensures that 
	$\| \nabla J_{i} (\theta_{t}) \|_{2} \le g$ 
	holds true for some constant $g > 0$.
	Finally, applying the fact that the stochastic approximation $u_{t}$ is a 
	quasi-stationary estimate of the term $\mathbb{E}[\delta \phi]$, cf.
	\cite{sutt:nips08}, we have
	\begin{equation}
	\begin{split}
		\mathbb{E}\big[ \widetilde{\nabla} J_{3} (\theta_{t}) \big] =&~\! 
		\mathbb{E}\big[ (\gamma \phi'_{t} - \phi_{t}) \phi_{t}^{\top} u_{t} \big] \\
		=&~\! \mathbb{E}\big[ (\gamma \phi'_{t} - \phi_{t}) \phi_{t}^{\top} \big]
		\mathbb{E}\big[ \delta_{t} \phi_{t} \big] \\
		=&~\! \nabla J_{3}(\theta_{t}).
	\end{split}
	\end{equation}
	Then the result follows from Theorem~\ref{thm:01}.
	\vspace{-3mm}
\end{proof}

%
In order to minimize the MSPBE function $J_{2}$, two efficient GTD algorithms are 
developed in \cite{sutt:icml09}. Their approximate stochastic updates are
defined as
\begin{subequations}
\label{eq:gtd_ist2}
\begin{align}
	\widetilde{\nabla} J_{2}^{(1)}(\theta_{t}) = (\phi_{t}^{\top } w_{t}) 
	(\gamma \phi_{t}' - \phi_{t}), \label{eq:gtd2_ist} \\ 
	\widetilde{\nabla} J_{2}^{(2)}(\theta_{t}) =  
	\gamma (\phi_{t}^{\top} w_{t}) \phi_{t}'  
	- \delta_{t} \phi_{t}, \label{eq:tdc_ist}
\end{align}
\end{subequations}
where
\begin{equation}
	w_{t+1} = w_{t} + \beta_{t} (\delta_{t} - \phi_{t}^{\top} w_{t}) \phi_{t}.
\end{equation}
We refer to the corresponding $\ell_{1}$ regularized GTD algorithms, which employ the 
updates \eqref{eq:gtd2_ist} and \eqref{eq:tdc_ist}, 
as GTD2-IST and TDC-IST algorithms, respectively.
With no surprises, they share similar convergence properties as the 
GTD-IST algorithm.

\begin{corollary}
\label{cor:02}
	If $(\phi_{t},r_{t},\phi'_{t})$ is an i.i.d sequence with uniformly
	bounded second moments, and both $\mathbb{E}[\phi (\phi - \gamma \phi')^{\top}]$ and 
	$\mathbb{E}[\phi \phi^{\top}]$ are invertible, then both the GTD2-IST and the 
	TDC-IST algorithms, whose updates are specified in \eqref{eq:gtd_ist2}, 
	converge with probability one to the TD solution.
\end{corollary}
\begin{proof}
	The $\ell_{1}$ regularized MSPBE cost function $F_{2}$ can be written as
	\begin{equation}
	\begin{split}
		F_{2}(\theta) =~\! & \mathbb{E}[\delta_{\theta } \phi]^{\top} 
		\mathbb{E}[\phi \phi^{\top}]^{-1} \mathbb{E}[\delta_{\theta} \phi] + 
		\eta \| \theta \|_{1} \\
		=~\! & \mathbb{E}\big[r \phi + \theta^{\top} (\gamma \phi' - \phi) \phi \big]^{\top}
		\mathbb{E}[\phi \phi^{\top}]^{-1}
		\mathbb{E}\big[r \phi + \theta^{\top} (\gamma \phi' - \phi) \phi \big] + 
		\eta \| \theta \|_{1}.
	\end{split}
	\end{equation}
	The function $F_{2}$ is strictly convex 
	if the matrix 
	\begin{equation}
		\mathbb{E}\big[\phi (\gamma \phi' - \phi)^{\top}\big]
		\mathbb{E}[\phi \phi^{\top}]^{-1}
		\mathbb{E}\big[(\gamma \phi' - \phi) \phi^{\top}\big]
	\end{equation}
	is positive definite, i.e. both $\mathbb{E}[\phi (\phi - \gamma 
	\phi')^{\top}]$ and $\mathbb{E}[\phi \phi^{\top}]$ are invertible.
	By the fact that the stochastic approximation $w_{t}$ is a 
	quasi-stationary estimate of the term $\mathbb{E}[\phi \phi^{\top}]^{-1}
	\mathbb{E}\big[(\gamma \phi' - \phi) \phi^{\top}\big]$, cf.
	\cite{sutt:icml09}, we get 
	\begin{equation}
		\mathbb{E}\big[ \widetilde{\nabla} J_{2}^{(1)} (\theta_{t}) \big] 
		= \mathbb{E}\big[ \widetilde{\nabla} J_{2}^{(2)} (\theta_{t}) \big]
		= \nabla J_{2}(\theta_{t}).
	\end{equation}
	Then, the result follows straightforwardly from the same arguments as 
	in Corollary~\ref{cor:01}.
	\vspace{-5mm}
\end{proof}


\begin{figure}[t!]
\begin{center}
	\subfigure[GTD ($\alpha = 0.1, \beta = 0.01$)]{
	\includegraphics[width=0.47\textwidth]{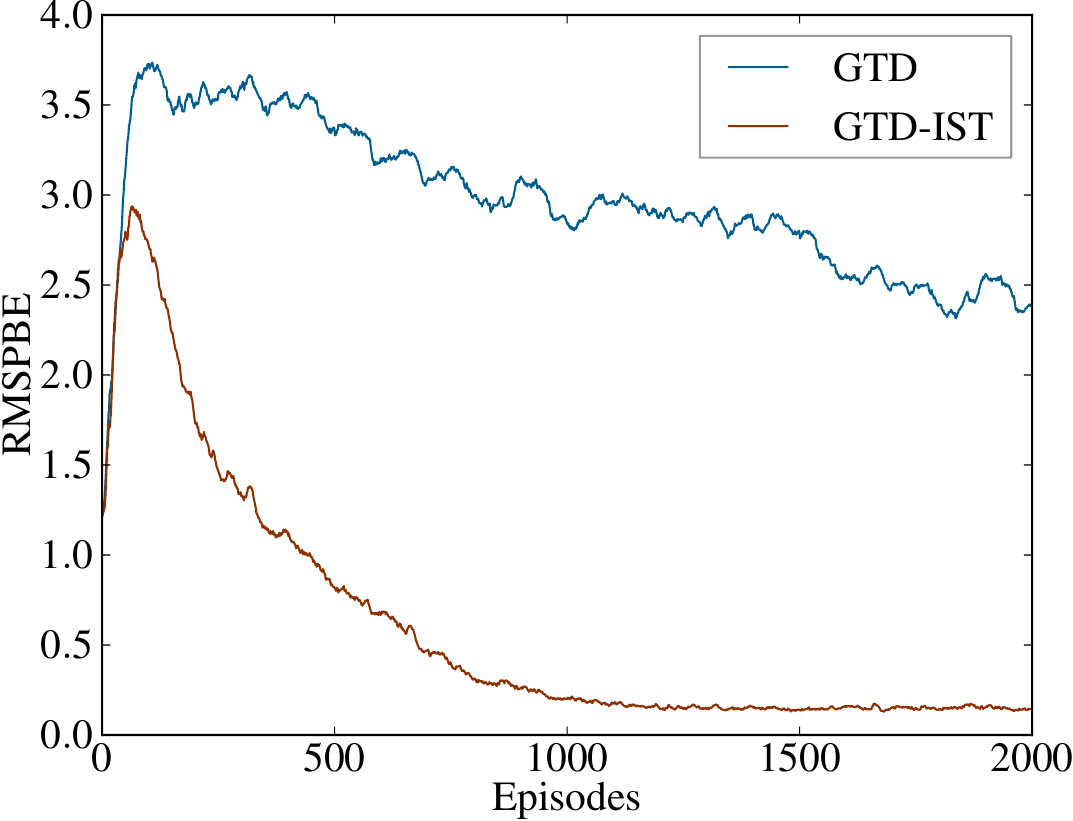}
	\label{fig:rmspbe_gtd}
	}
	\subfigure[GTD2 ($\alpha = 0.1, \beta = 0.1$)]{
	\includegraphics[width=0.47\textwidth]{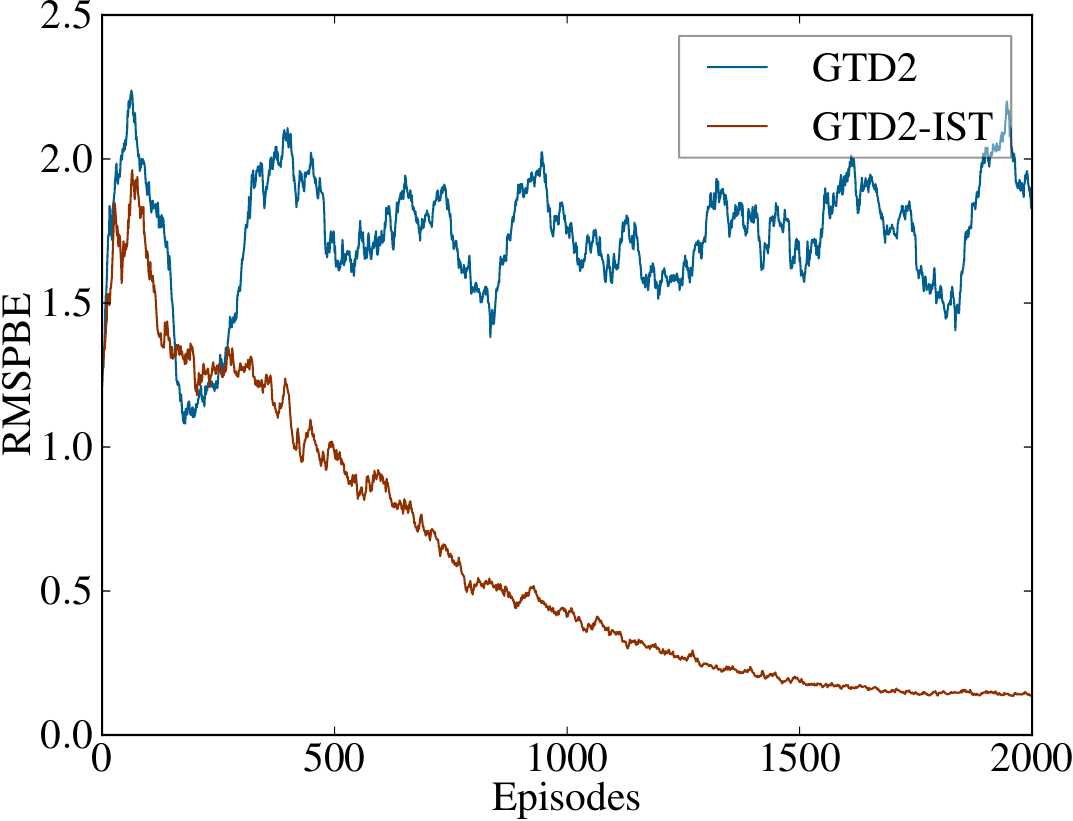}
	\label{fig:rmspbe_gtd2}
	}
	\subfigure[TDC ($\alpha = 0.1, \beta = 0.05$)]{
	\includegraphics[width=0.47\textwidth]{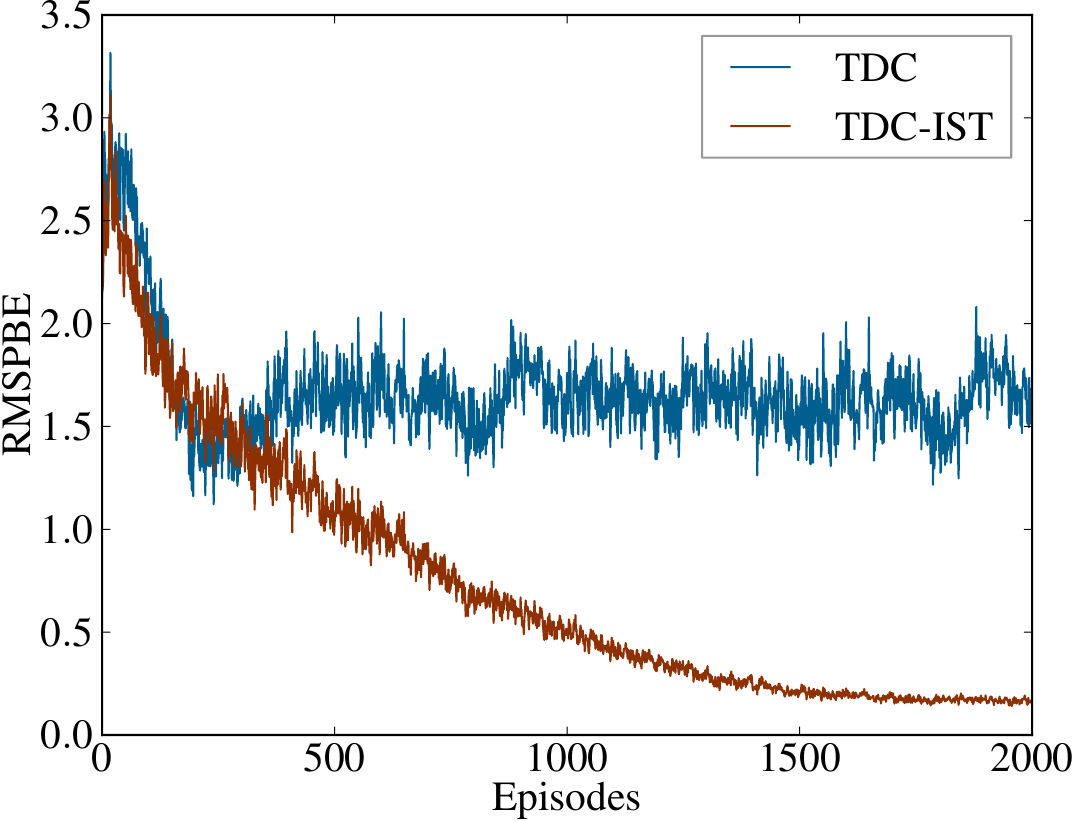}
	\label{fig:rmspbe_tdc}
	}
    \caption{A comparison of IST based GTD learning family ($\eta = 0.001$).}
    \label{fig:rmspbe}
    \vspace*{-7mm}
\end{center}
\end{figure}   

\section{Numerical Experiments}
\label{sec:04}
In this section, we investigate the performance of our 
proposed $\ell_{1}$ regularized GTD algorithms, compared with two existing 
$\ell_{1}$ regularized TD algorithms, in both the on-policy and off-policy settings. 

\begin{figure}[t!]
\begin{center}
    \includegraphics[width=0.47\textwidth]{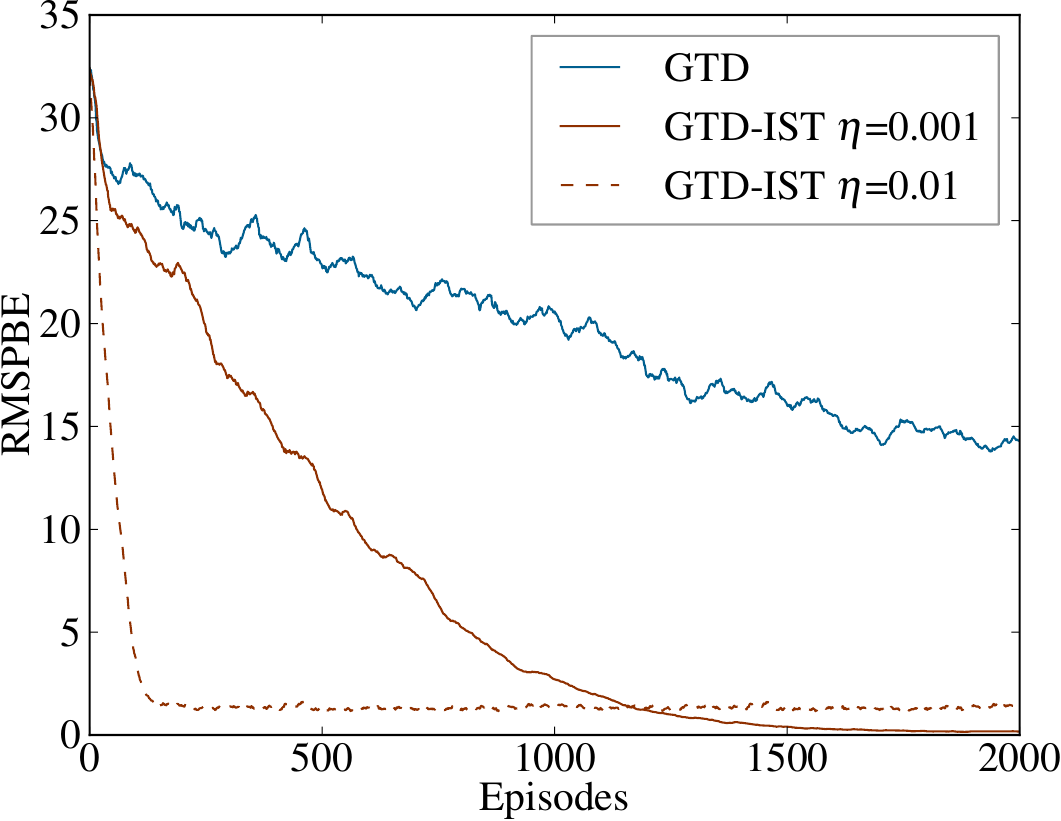}
    \caption{GTD with unfavorable initializations ($\alpha = 0.1, \beta = 0.01$).}
    \label{fig:badinit}
\end{center}
\end{figure}

\subsection{Experiment One: On-Policy Learning}
In this experiment, we apply our proposed algorithms to a random walk 
problem in the chain environment consisting of seven states.
There exists only one action and the transition probability 
of going right or left is equal. 
A reward of one is only assigned in the rightmost state, which is the terminal state,
whereas the rewards are zero everywhere else.
The features consist of a binary encoding of the states and 
ten additional ``noisy'' features, which are simply Gaussian noise.
In this setting, we run three different experiments. 
%

\subsubsection{Regularized vs. Un-regularized}
This experiment compares the performance of the proposed $\ell_{1}$ 
regularized GTD algorithms with their un-regularized counterparts.
Figure $\ref{fig:rmspbe}$ shows the learning curves of three GTD 
learning algorithms, namely, GTD, GTD2, and TDC, together with their
regularized versions.
It is evident that IST based GTD algorithms outperform all their original 
un-regularized versions respectively.
The experimental results demonstrate the effectiveness of IST based 
GTD learning algorithms.

\subsubsection{Unfavorable Initializations}
The second experiment investigates the recovery behavior 
and convergence speed of our proposed algorithms with unfavorable initializations. 
Here, we only consider the simple GTD-IST algorithm. 
The parameter vector $\theta$ is initialized to have ones for 
all the noisy features and zeros for all the ``good'' features. 
In other words, our experiment starts with the initialization of selecting all
the ``bad'' features. 
The results in Figure~\ref{fig:badinit} show that the $\ell_{1}$
regularized GTD algorithms, i.e. the GTD-IST algorithm with different parameter value
$\eta$, converge faster to the correct selection of features 
than the original GTD algorithm.

\subsubsection{GTD-IST Algorithms vs. Others}
In the third experiment, we compare the GTD-IST algorithms with  
the L1TD algorithm from \cite{pain:tr12} and the LARS-TD algorithm from
\cite{kolt:icml09}. 
Results in both Figure~\ref{fig:noise_without} and \ref{fig:noise_with}
imply that, with or without noise, 
all three GTD-IST algorithms outperforms the L1TD algorithm consistently.
A closer look at the result in the zoomed-in window in 
Figure~\ref{fig:noise_with_zoom} shows that the LARS-TD algorithm performs
the best with the presence of noise.
This might be due to the fact that the LARS-TD algorithm updates, after every 
20 episodes, using all the samples available.
Nevertheless, without any surprise, a timing experiment shows in Table~1 that 
the LARS-TD algorithm performs much slower 
than the other online algorithms.
\begin{figure}[t!]
\begin{center}
	\subfigure[without noise]{
	\includegraphics[width=0.47\textwidth]{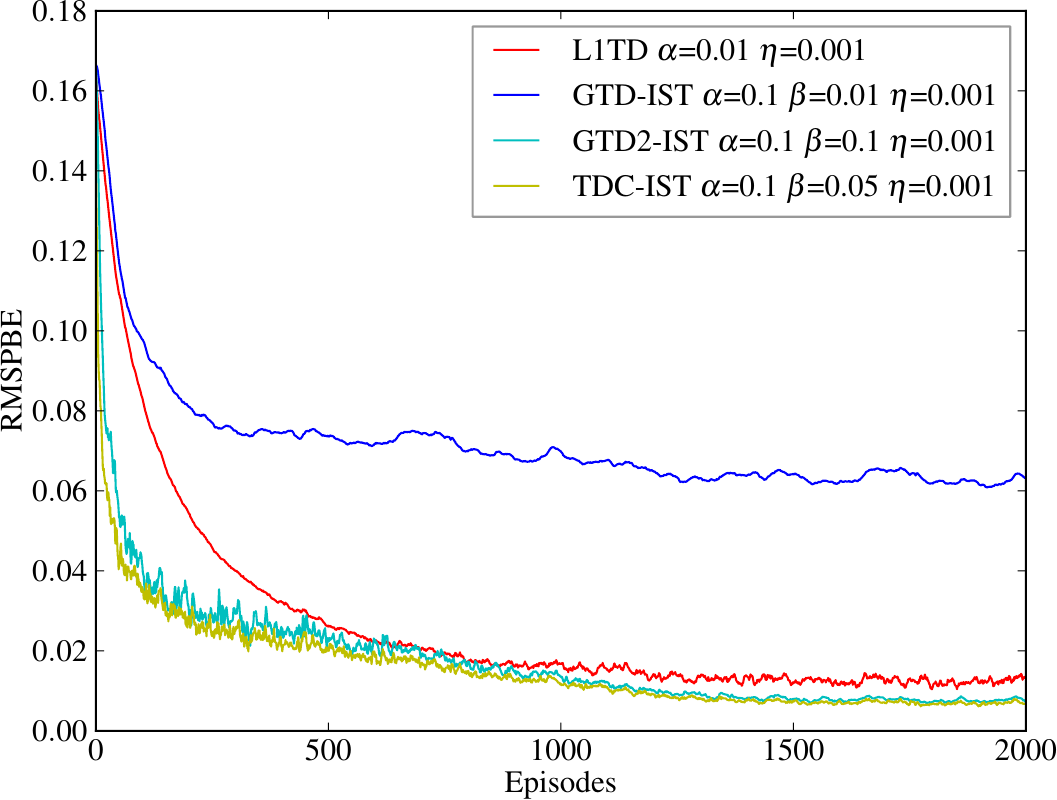}
	\label{fig:noise_without}
	}
	\subfigure[with noise]{
	\includegraphics[width=0.47\textwidth]{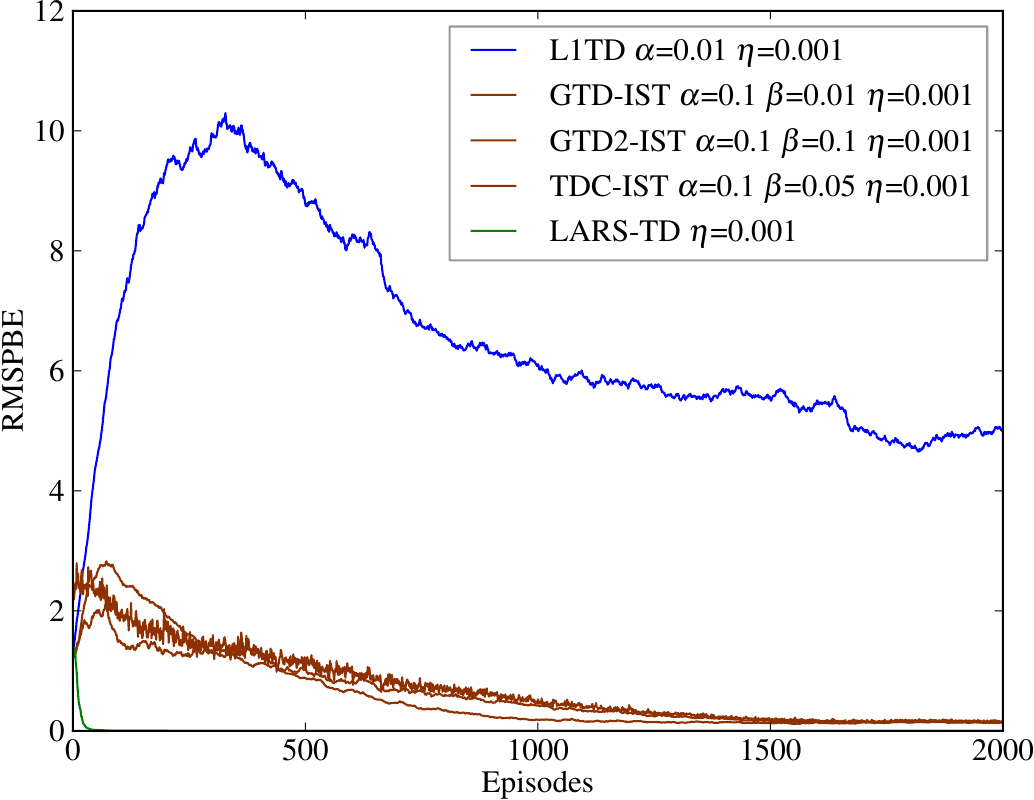}
	\label{fig:noise_with}
	}
	\subfigure[zoomed-in window on \ref{fig:noise_with}]{
	\includegraphics[width=0.47\textwidth]{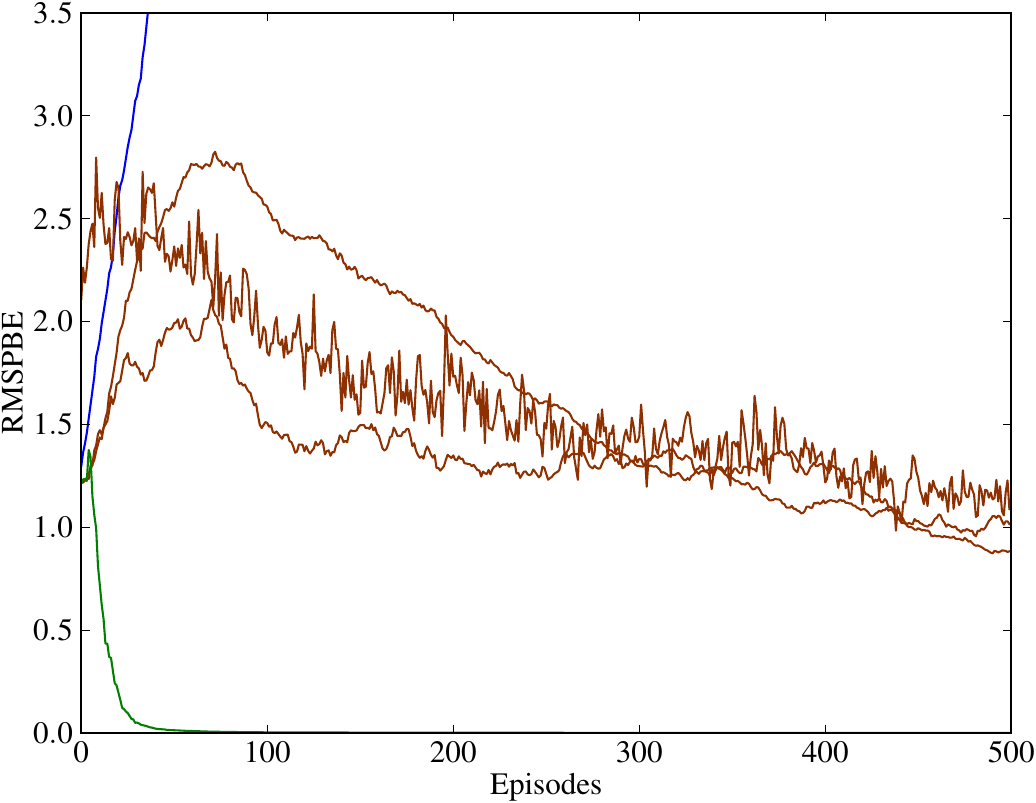}
	\label{fig:noise_with_zoom}
	}
    \caption{Performance of LARS-TD, L1TD, and the GTD-IST algorithms.}
    \label{fig:noise}
    \vspace*{-7mm}
\end{center}
\end{figure}

\begin{table}[h!]
	\centering
	\begin{tabular}{|p{2cm}|p{2cm}|p{2cm}|p{2cm}|p{2cm}|p{2cm}|}
	\hline
		& L1TD & GTD-IST & GTD2-IST & TDC-IST & LARS-TD \\
	\hline
		Time (s) & \hspace{5mm}9.2160 & \hspace{5mm}9.5565 & \hspace{5mm}8.2130 & \
		\hspace{4mm}8.4660 & \hspace{3mm}118.5490 \\
	\hline
	\end{tabular}
	\caption{Time measurement of performing 2000 episodes.}
\end{table}


\subsection{Experiment Two: Off-Policy Learning}
To test the performance of the GTD-IST algorithms on the off-policy learning, 
we employ the well-known star example, proposed in \cite{bair:icml95}. 
It consists of seven states with one state being considered as
the ``center''. 
In each of the outer states, the agent can choose between two actions: 
either the ``solid'' action, which takes it to the center state with probability one, 
or the ``dotted'' action, which takes it to any of the other states with equal 
probability. 
Reward on all state transitions is equal to zero and the states are 
represented by tabular features as described in the original setting.
We add 20 noisy ``Gaussian'' features to the state
representation. 
The behavior policy chooses the ``solid'' action 
with the probability $1/7$ and the ``dotted'' otherwise,
while the estimation policy chooses always the ``dotted'' action.
%
The learning curves in Figure \ref{fig:offpolicy} shows that both GTD-IST and GTD2-IST
algorithms outperform their original counterparts consistently.

\begin{figure}[t!]
\begin{center}
    \includegraphics[width=0.47\textwidth]{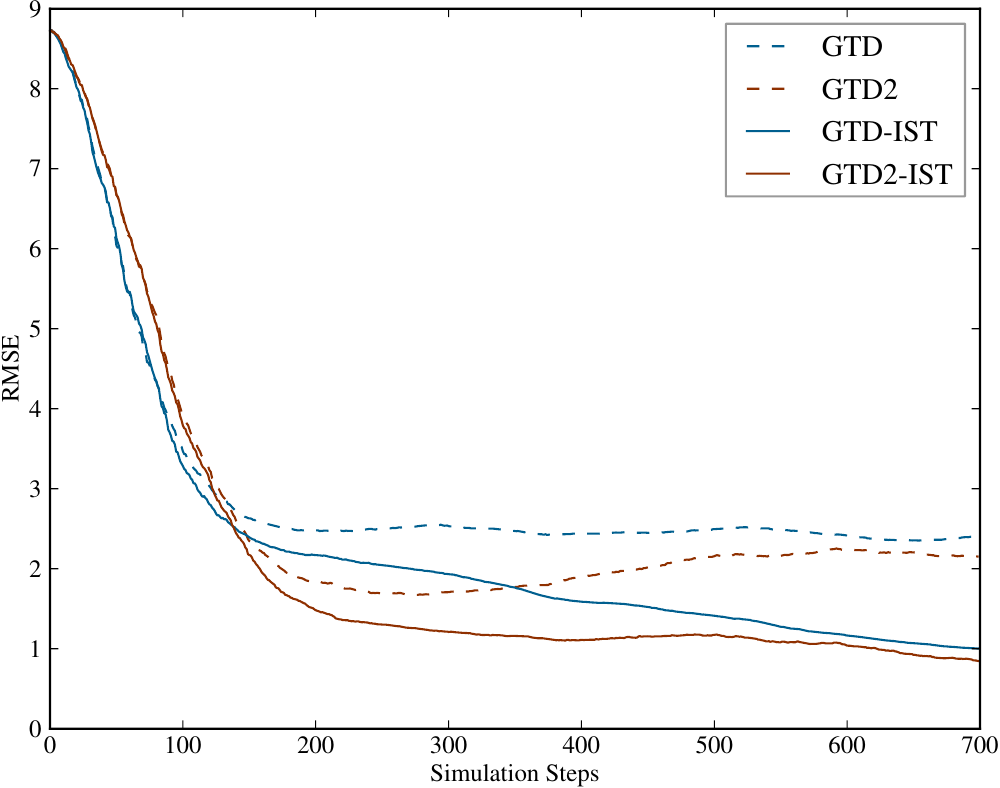}
    \vspace{-3mm}
    \caption{Off-policy example ($\alpha = 0.01, \beta = 0.1, \eta = 1$).}
    \label{fig:offpolicy}
\end{center}
\end{figure}



%

\section{Conclusions}
\label{sec:05}
This work combines the recently developed GTD methods with $\ell_{1}$
regularization, and proposes a family of GTD-IST algorithms.
We investigate the convergence properties of the proposed algorithms from
the perspective of stochastic optimization.
Preliminary experiments demonstrate that the proposed family of GTD-IST algorithms
outperform all their original counterparts and two existing $\ell_{1}$ regularized 
TD algorithms.
Being aware of advanced developments in the community of sparse representation, 
we project to employ further state-of-the-art algorithms of sparse representation
to RL.
For example, the IST algorithms are usually known to be slow compared to 
other advanced $\ell_{1}$ minimization algorithms.
Applying more efficient $\ell_{1}$ minimization algorithms, such as \cite{beck:siims09}, 
to TD learning are of great interests as the future work.
%

\section*{Acknowledgements}
This work has been partially supported by the International 
Graduate School of Science and Engineering (IGSSE), Technische
Universit\"at M\"unchen, Germany.
The authors would like to thank Christopher Painter-Wakefield for providing us 
with the Matlab implementation of the L1TD algorithm.

\bibliography{shenbib}

\end{document}